\documentclass[lettersize,journal]{IEEEtran}

\usepackage{amsmath,amsfonts}
\usepackage{algorithmic}
\usepackage{algorithm}
\usepackage{balance}
\usepackage{array}
\usepackage[caption=false,font=normalsize,labelfont=sf,textfont=sf]{subfig}
\usepackage{textcomp}
\usepackage{stfloats}
\usepackage{url}
\usepackage{verbatim}
\usepackage{graphicx}
\usepackage{cite}
\usepackage{acronym}
\usepackage{orcidlink}
\usepackage{colortbl}
\usepackage{threeparttable}
\usepackage{xcolor}
\usepackage[caption=false,font=normalsize,labelfont=sf,textfont=sf]{subfig}
\usepackage{amsthm}

\colorlet{RED}{red}
\colorlet{BLACK}{black}

\newcommand{\figref}[1]{Fig. \ref{#1}}
\newcommand{\secref}[1]{Section \ref{#1}}
\newcommand{\tabref}[1]{Tab. \ref{#1}}
\renewcommand{\quote}[1]{``#1''}

\DeclareMathOperator*{\argmin}{\arg\min}

\definecolor{layer400_800}{HTML}{FF8C00}
\definecolor{layer800_512}{HTML}{00FF00}
\definecolor{layer512_256}{HTML}{FF00FF}
\definecolor{layer256_10}{HTML}{FFFF00}
\definecolor{min_network}{HTML}{00BFFF}

\begin{document}
\acrodef{ai}[AI]{Artificial Intelligence}
\acrodef{ml}[ML]{Machine Learning}
\acrodef{snn}[SNN]{Spiking Neural Network}
\acrodef{cnn}[CNN]{Convolutional Neural Network}
\acrodef{ann}[ANN]{Artificial Neural Network}
\acrodef{shd}[SHD]{Spiking Heidelberg Digits}
\acrodef{ssc}[SSC]{Spiking Speech Commands}
\acrodef{sota}[SOTA]{state-of-the-art}
\acrodef{lif}[LIF]{Leaky-Integrate and Fire}
\acrodef{lod}[LOD]{Leading-One Detector}
\acrodef{lopd}[LOPD]{Leading-One Position Detector}
\acrodef{noc}[NoC]{Network-on-Chip}
\acrodef{fpga}[FPGA]{Field-programmable Gate Array}
\acrodef{aer}[AER]{Address Event Representation}
\acrodef{orf}[ORF]{OpenROAD-Flow-Script}
\acrodef{pdk}[PDK]{Process Design Kit}

\newtheorem{proposition}{Proposition}

\title{Lightweight LIF-only SNN accelerator using differential time encoding}

\author{Daniel Windhager$^\ast$\orcidlink{0009-0007-6443-6515}, Lothar Ratschbacher$^\dagger$\orcidlink{0000-0002-2631-0977}, Bernhard A. Moser$^{\ast\ast\ddagger}$\orcidlink{0000-0003-1859-046X}, Michael Lunglmayr$^\S$\orcidlink{0000-0002-4014-9681}
\thanks{$\hspace{-1em}^\ast$Intelligent Wireless Systems, Silicon Austria Labs, Linz, Austria, daniel.windhager@silicon-austria.com\\
$^\dagger$Intelligent Wireless Systems, Silicon Austria Labs, Linz, Austria, lothar.ratschbacher@silicon-austria.com\\
$^\ddagger$Institute of Signal Processing, Johannes Kepler University, JKU SAL IWS Lab, Linz, Austria, bernhard.moser@jku.at\\
$^{\ast\ast}$double affiliation: Software Competence Center Hagenberg (SCCH), 4232 Hagenberg, Austria\\
$^\S$Institute of Signal Processing, Johannes Kepler University, Linz, Austria, michael.lunglmayr@jku.at}
}

\markboth{IEEE TRANSACTIONS ON VERY LARGE SCALE INTEGRATION (VLSI) SYSTEMS}%
{Windhager \MakeLowercase{\textit{et al.}}: \title}

\maketitle

\begin{abstract}
\acp{snn} offer a promising solution to the problem of increasing computational and energy requirements for modern \ac{ml} applications. Due to their unique data representation choice of using spikes and spike trains, they mostly rely on additions and thresholding operations to achieve results approaching \ac{sota} \acp{ann}. This advantage is hindered by the fact that their temporal characteristic does not map well to already existing accelerator hardware like GPUs. Therefore, this work will introduce a hardware accelerator architecture capable of computing feedforward LIF-only \acp{snn}, as well as an accompanying encoding method to efficiently encode already existing data into spike trains. Together, this leads to a design capable of~$>\negmedspace99\%$ accuracy on the MNIST dataset, with $\approx{}0.29ms$ inference times on a Xilinx Ultrascale+ FPGA, as well as $\approx{}0.17ms$ on a custom ASIC using the open-source predictive 7nm ASAP7 PDK. Furthermore, this work will showcase the advantages of the previously presented differential time encoding for spikes, as well as provide proof that merging spikes from different synapses given in differential time encoding can be done efficiently in hardware. 
\end{abstract}


\section{Introduction}
\IEEEPARstart{T}{he} field of \ac{ai} and more specifically \ac{ml} is ever changing and growing, with a clear trend towards bigger and more expansive network architectures, which in turn means larger parameter spaces and datasets. While GPUs, faster memory and higher bandwidths can curb some of these issues, the fact remains that the most recent architectures, like the transformer-based GPT-4~\cite{gpt4}, currently require large server infrastructure, for inference as well as for training. Because of the large computational effort involved, less computationally expensive architectures have gained interest in academia and industry in recent years. One of these approaches is to utilize \acp{snn} instead of traditional \acp{ann}, due to their improved power efficiency~\cite{snn_vs_ann}, sparsity characteristics~\cite{Moser_2025} and, depending on the choice of the neuron model and spike encoding, decreased computational effort. This is especially important for embedded devices and edge applications, as these often impose hard limits on energy efficiency and inference time. Although they offer a promising solution, \acp{snn} introduce a different set of problems, such as the need for choosing an appropriate encoding method for the data. Even more important, GPUs and CPUs are not well suited to the task of accelerating \acp{snn} due to their asynchronous and event-based architecture, making efficient custom accelerator architectures essential. An additional issue is the encoding of vectorized data into spike sequences, which is less than obvious for \acp{snn}. This issue of choosing the encoding can be mitigated by incorporating the encoding into the training process, thus allowing the network to adjust based on the training data. Furthermore, the custom hardware architecture presented here, which improves upon a previous design~\cite{aicas2024}, allows for efficient processing of \acp{snn}, especially on resource constrained \acp{fpga}. 
\\
This work will extend both of these previously published concepts~\cite{iscas2024}, by providing a mathematical proof of the main spike merger network structure, as well as providing justification for the differential time encoding chosen in this and previous works. Additionally, this work will highlight the universality of the concept of learned encoding, by showcasing its performance for the MNIST~\cite{mnist} and CIFAR10~\cite{cifar10} dataset. Finally, this work also provides synthesis results for the presented accelerator architecture for the Xilinx Ultrascale+ ZCU102 MPSoC board as well as for the open-source predictive 7nm ASAP7 PDK~\cite{ASAP7}. These results demonstrate that the proposed architecture can be implemented efficiently on resource constrained digital hardware while still being competitive in terms of accuracy and inference speed.

\section{Background and related work}
Many different works regarding \acp{snn} have been proposed in the past, dealing with the theoretical and practical machine learning aspects as well as different architectures and designs capable of inference or even training the network using custom hardware. Although quite a few publications explore training and inference on GPUs~\cite{hao2024,qu,5955897,yanchenli,gpuranc}, many only train on GPUs to then deploy the trained model on custom hardware such as \acp{fpga} or even ASICs. A very prominent example is Intel's Loihi and Loihi 2~\cite{original_loihi} architecture, with the second generation being capable of simulating more than 1 million neuron cores on the chip, with fully programmable neuron models and flexible routing between the neurons using a \ac{noc}. Loihi is, contrary to many other architectures, even capable of on-chip learning, making it possible to use the same chip for training as well as inference. Other ASICs capable of \ac{snn} inference include the SpiNNaker~2~\cite{spinnaker,spinnaker2} which utilizes 152 parallel ARM cores with custom dedicated accelerators, connected by a flexible \ac{noc} to enable fast processing of sparse and non-sparse networks. While Loihi\footnote{Note that, at the time of writing Loihi is only available for research purposes by directly contacting the Intel Neuromorphic Research Community (NRC).} and SpiNNaker are designed to handle various different networks and neuron models, there are also designs which optimize with a special use-case in mind, such as SparrowSNN~\cite{sparrowsnn}. It is designed specifically for detecting heart disease using the MIT-BIH dataset~\cite{mit-bih} and a custom 22nm ASIC for acceleration. Although there exist many ASIC designs specifically for acceleration of \acp{snn}, even more designs use \acp{fpga} for this task, due to their lower cost for small quantities and their reprogrammability. While many of them focus on energy-efficiency as the main selling criteria of \acp{snn}~\cite{time_domain_neurons,hybrid_snn,snn_audio,ali2024}, very few architectures focus on \quote{pure} \acp{snn}, i.e. networks which only include \ac{lif} neurons. In fact, many publications combine regular \acp{ann}, such as \acp{cnn} or transformer based architectures with an activation layer consisting of \ac{lif} neurons~\cite{deepfire2,liu_hanwen,firefly,syncnn,s2n2,nevarez}, which somewhat diminishes the power efficiency aspect, since multipliers are once again necessary. To completely remove the necessity for multiplications, the neuron model must be chosen such that it can be computed with addition and shift operations only, which is why the authors of this work chose to realize their \ac{snn} networks using only an implementation-efficient variant of the \ac{lif} model. Additionally, many works that have been published in recent years use the \ac{aer} to represent spikes~\cite{liu_pan,mesquida,kim_park,madhuvanthi}, resulting in a large overhead (specifically due to the choice of using absolute time to represent spikes), impairing power and computational efficiency. Here, as described below, we argue for using a more efficient differential time encoding.

\section{Efficient spike representation}
Many of the works using the \ac{aer}, use absolute time to represent the time at which a spike or event occurred. However, this is not only suboptimal but can also lead to problems, especially for systems running continuously, as after some finite time the counter used to represent the current time will overflow. As an alternative, the authors proposed the differential time representation in~\cite{aisys} that can be shown to be almost equivalent, in terms of required bit length, to the absolute time representation in the worst case, whereas for most other cases it substantially reduces the bit width required to represent the same information. This is accomplished by encoding the time difference between a spike and its predecessor, instead of using the absolute time information. To represent spikes that occur at a time not representable with the specified bit width of the differential time, an overflow symbol is introduced (the highest representable number $2^b-1$, where $b$ is the chosen bit width), which signifies the passage of time without a spike.
\\\\
To illustrate why the AER is not ideal, consider a spike train with $K\in\mathbb{N}$ spikes for which the last spike occurs at some time $T\in\mathbb{N}_0$. Using absolute time, $\lceil\log_2(T+1)\rceil$ bits are needed to represent every individual spike, whereas the differential time encoding depends on the distribution of the times between the spikes (that are typically smaller than the maximum time). Assuming for the moment that we do not wish to use more than one symbol to represent a single spike, i.e. we do not allow overflow symbols, then the worst case, where all $K$ spikes occur at the latest possible moment, would also require a bit width of $\lceil\log_2(T+1)\rceil$. The best case for the differential time encoding method occurs when all $K$ spikes are exactly evenly spaced, reducing the required bit width to $\lceil\log_2(\frac{T+1}{K})\rceil$, thus giving a substantial improvement. The exact savings in terms of bit width depend heavily on the distribution of the spikes, however, even in the worst case the differential time representation at most devolves into the same performance as the absolute time representation. Notice that in the current discussion no overflow symbols have been considered. As will be shown next, the overflow symbols actually improve the efficiency of the differential time representation in terms of the required bit width.

\subsection{Sampled time compared to differential time}
Another possible way to represent spike trains is to simply sample the spike trains and use the resulting sequences of ones and zeros (assuming only positive spikes of amplitude one are used), which could be represented as 
\begin{equation}
    S = (s_1, s_2, ..., s_{n-1}), s_i \in \{0, 1\}.
\end{equation}
This sampled version of the spike train is actually equivalent to the differential time representation, using only a single bit, with 1 as the spike amplitude, whereas 0 can be interpreted as overflow symbol signifying that one sampling time has passed where no spike was triggered. In the sampled version, a 0 signifies that at this point in time no spike occurred whereas a 1 signifies that a spike occured. In contrast, when using the differential time representation with a single bit, the value 0 indicates that a spike occurs in zero time, i.e. now, whereas the value 1 is the highest representable value and is thus used as the overflow symbol, signifying that no spike occurs at this moment in time. 
\\\\
Clearly, if spikes occur less often, a larger bit width for the differential time representation would yield an advantage over the sampled time representation, but it is not immediately obvious which bit width is ideal. Intuitively, a bit width that is too low for a given spike train produces a large number of overflow symbols thus reducing efficiency, whereas an exceedingly large bit width would never produce overflow symbols but wastes many bits on encoding a smaller time difference, as described previously. The most efficient bit width is dependent on the distribution of spikes in the spike trains one wants to encode. To be more precise, the overall distribution of the time differences between the spikes, determines the ideal bit width to choose. In order to properly optimize for this, we need a cost function to compare costs of different approaches. One candidate for this is to simply use $N_S b$, where $N_S$ is the number of symbols required to represent the entire spike train and $b$ is the bit width of a single symbol. Using this cost function, it is simple to arrive at the optimal cost of encoding the spike train, by realizing that every difference $d_i\in\mathbb{N}_0$ for which $0 \le d_i < 2^b-1$ holds, can be represented with a single symbol. Similarly, every difference $d_i$ for which $2^b-1 \le d_i < 2^{b+1}-1$ holds, can be represented with two symbols, and so forth. In general, the number of symbols for a given time difference and bit width is $\left\lceil d_i / (2^b-1)\right\rceil$. Therefore the total number of symbols for the entire spike train is 
\begin{equation}
    N_S(b) = \sum_{i} \left\lceil \frac{d_i}{2^b-1} \right\rceil
\end{equation}
leading to the overall bits required to encode the entire spike train as
\begin{equation}
    B_S(b) = bN_S(b) = b\sum_{i} \left\lceil \frac{d_i}{2^b-1} \right\rceil.
    \label{eq:minimum}
\end{equation}
Ideally we would now like to find
\begin{equation}
    b^* = \argmin_{b\in\mathbb{N}} B_S(b)
    \label{eq:opt}
\end{equation}
which, due to the ceiling function, is unfortunately not possible to find analytically. It is however possible, and quite straightforward, to find a minimum by computing all values of $B_S(b)$ for practically relevant $b$-values. To do this, we first generate a histogram of all the differential time spikes that we wish to encode. An example of all the differential time spikes that occur in a 400-128-10 network using the MNIST test dataset can be seen in \figref{fig:histogram}. Instead of having to iterate through all of the differences as given in \eqref{eq:minimum}, we now simply calculate the number of symbols required to encode each of the differential times in the histogram for a given bitwidth, which yields
\begin{equation}
    B_S(b) = b\sum_iK_i\left\lceil\frac{i}{2^b-1}\right\rceil 
\end{equation}
where $K_i$ is the number of differential time spikes with difference $i$ as seen in \figref{fig:histogram}. The resulting plot of $B_S(b)$ for different values of $b$ can be seen in \figref{fig:ideal_bitwidth}, which shows that the bit width for the differential time encoding leading to the least amount of bits required, is 2. 
\begin{figure}
    \centering
    \includegraphics[width=\linewidth]{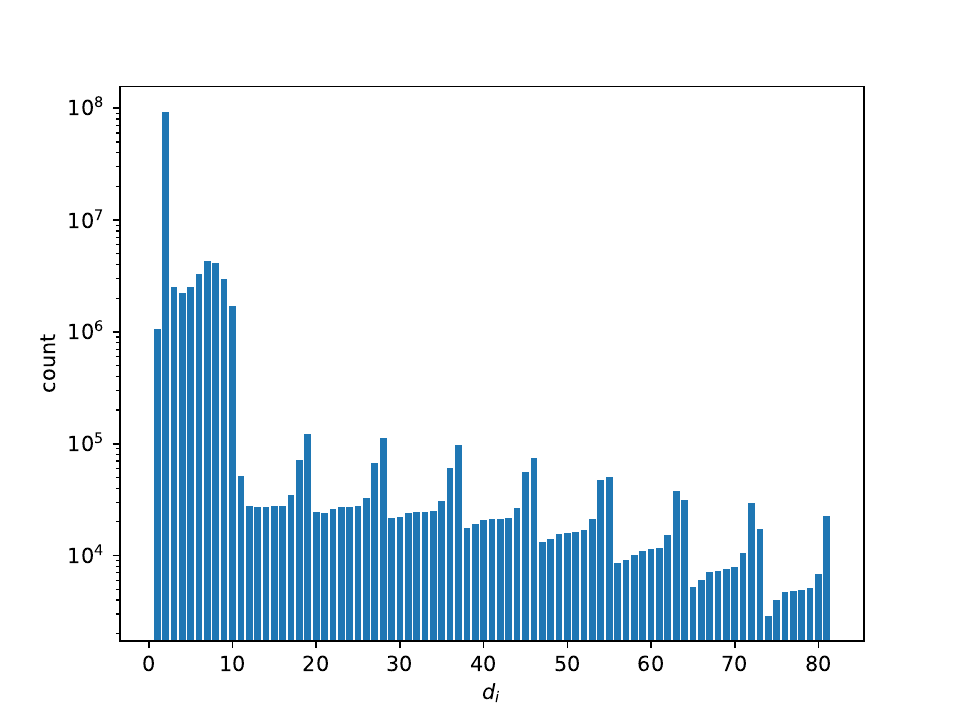}
    \caption{Histogram for the difference times when using the test set of the MNIST dataset with 9x9 learned encoding as shown in \secref{sec:learned_encoding}.}
    \label{fig:histogram}
\end{figure}
~
\begin{figure}
    \centering
    \includegraphics[width=\linewidth]{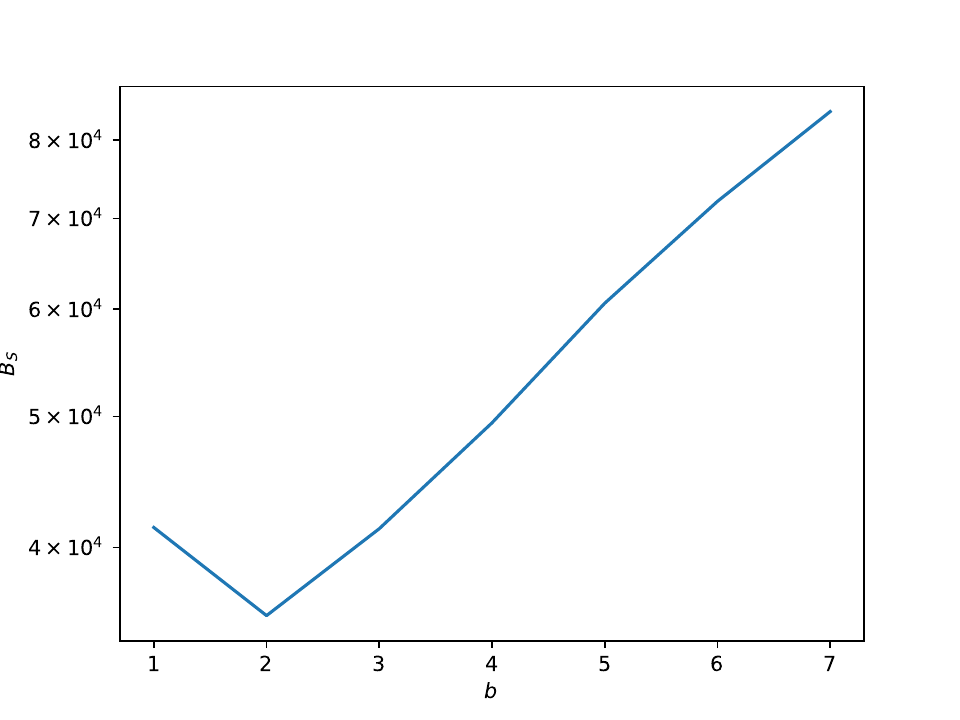}
    \caption{Overall number of bits $B_S$ required to encode the entire spike train based on chosen bit width $b$ for the differential time.}
    \label{fig:ideal_bitwidth}
\end{figure}

When analyzing the three discussed encodings: absolute, sampled or differential, we gave reasons in terms of efficiency advocating for differential time encoding. Considering implementation of neuron processing, one might have reservations for differential time encoded representations as one has to determine across different spike trains which spike to process next in chronological order. Although this is trivial for absolute time encoded and sampled sequences, it is not obvious for differential time encoded spike trains where ordering is given within a spike sequence, but is not guaranteed across different sequences. However, as we show in the next section, such order can be restored by a simple procedure without the need for conversion to other representations.
\section{Comparing spike trains in differential time}
\label{sec:spike_sorter}
In the following, we show that multiple differential time encoded spike sequences can be merged without calculating the cumulative sums first and subsequently merging the sequences in absolute time. We describe the proof for two sequences, since it can trivially be extended to any number of sequences.

Let $ (A =\{a_i \}_{i=1}^{n}, \leq)$  be the ordered set of the absolute spike times of a spike train (the most commonly used ordering of ordered sets is the natural ordering of time values), and $(B=\{b_j \}_{j=1}^{m}, \leq)$ an ordered set of absolute spike times of a second spike train, defined in the same way. Note that for clarity all spikes and spike trains given in differential time in this section will be prefixed by $\Delta{}$ in order to avoid confusion. All other spikes are assumed to be given in absolute time unless stated otherwise.

The ordered sets of differential times are then defined such that $(\Delta A =\{\Delta a_i = a_i-a_{i-1}\}_{i=1}^{n}, a_{i-1} \leq a_i)$ using $a_0 = 0$.

The merged sequence $Y = {A \cup B}$ is realized by the order relation. Its ordered set of differential times $\Delta Y$ is again defined in the same way as described for $\Delta{}A$ above. The cumulative sum of the elements of $\Delta Y$ will again give $Y$. 

\begin{proposition} 
    The ordered set $\Delta Y$ can be generated directly from $\Delta A$ and $\Delta B$ using the following procedure: Compare the first elements of the input sets $\Delta A$ and $\Delta B$ to each other. Append the minimum to the output set $\Delta Y$, remove the minimum from its respective set, and update the other first element (when using more than two sets: update all the other first elements) by subtracting the minimum from it. Repeat until all input sets are empty.
    \label{prop:merge_spikes}
\end{proposition}

\begin{proof}
    Proofing the above proposition can be easily done based on the observation that the first elements of differential time encoded sets are in absolute relation to each other as they can be interpreted as absolute times to a set zero time. See Fig.~\ref{fig:diffspikes} for depiction. An ordering with respect to their difference times is guaranteed only for the first elements of the differential time encoded sets. This means that the first delta time of the merged spike train $\Delta Y$ must be the minimal element among the first elements of the input spike trains. Updating the other first elements (in the general case of more than two spike trains) by subtracting this minimum time can be interpreted as moving the set zero timeline (i.e. resetting the timeline to a new shifted zero). After the subtractions, all first elements of the differential time-encoded input spike trains are in absolute relation again, although now with the shifted set time. Outputting the minimum elements gives the sequence $\Delta Y$.
\end{proof}

\begin{figure}
	\centering
	\includegraphics[width=.8\linewidth]{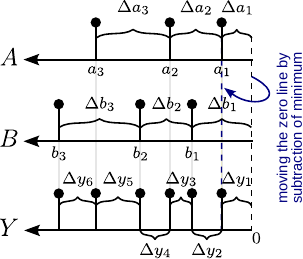}
	\caption{Merging of spike sequences given in differential time.}
	\label{fig:diffspikes}
\end{figure}

\section{Hardware architecture}
\label{sec:hw_arch}
Having established that the differential time representation has benefits compared to other representations, the next step is to explore how this can be used in implementations for \ac{snn} inference. This work is an extension of the design in~\cite{iscas2024}, which introduced a hardware architecture capable of accelerating feedforward \acp{snn} consisting of any number of layers of arbitrary size using the standard \ac{lif} neuron, while also using differential time spikes as inputs and outputs. The standard version of the \ac{lif} neuron is given as
\begin{align}
    \begin{split}
        P_{k} &= P_{k-1}\beta^{\Delta t} + \sum_i w_i s_i - s_{out}\\
        s_{out} &= \begin{cases}
        \mbox{$\theta$,} & \mbox{if } P \ge \theta\\
        \mbox{0,} & \mbox{if } P < \theta,
        \end{cases}
    \end{split}
    \label{eq:lif_neuron}
\end{align}
where $P\in\mathbb{R}$ is the neuron potential, $\beta\in[0,1]$ is the decay rate of the potential, $\Delta{}t\in\mathbb{N}_0$ is the time that passed since the last event, $k\in\mathbb{N}_0$ is the time where at least one input spike $s_i \in \{0, 1\}$ occurred, $\sum_i w_i s_i$ is the weighted sum of the input spikes (at the current event) and their corresponding weights, $s_{out}\in\mathbb{R}$ is responsible for resetting the neuron potential after a spike is triggered and $\theta\in\mathbb{R}$ is the threshold. The authors assume here that $\beta=0.5$, as it was shown by the same authors in~\cite{aicas2024} that this does not diminish the capability of the neuron, as well as $\theta=1$. Therefore, the hardware only needs to be able to store the continuously changing values of the neuron potential $P_k$ for each neuron, as well as the weights $w_i$, learned during the training stage of the network. As such, the hardware architecture needs a way to correctly find the earliest spike among all spike trains entering the network, as established in \secref{sec:spike_sorter}, in order to properly define $\Delta{}t$ needed for decaying the neuron potential in \eqref{eq:lif_neuron}.

\subsection{Merging the spikes in hardware}
This merging of the spikes can be accomplished by directly implementing Proposition~\ref{prop:merge_spikes}, leading to the spike merger element shown in \figref{fig:spike_merger}, with the variables $\Delta{}\hat{a}$ and $\Delta{}\hat{b}$ representing the continuously updated $\Delta{}a_i$ and $\Delta{}b_j$ respectively. This update is either by subtraction or by replacement with the next difference time of the spike train.

\begin{figure}
    \centering
    \includegraphics[width=0.8\linewidth]{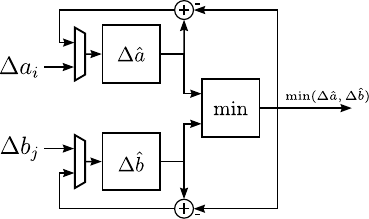}
    \caption{Spike merger element capable of merging two separate spike trains given in differential time.}
    \label{fig:spike_merger}
\end{figure}
The hardware block in \figref{fig:spike_merger} merges two separate spike trains given in differential time, producing a new spike train given in differential time, containing all spikes from both input spike trains. By cascading the spike merger elements, any number of spike trains can be efficiently combined into a single spike train, using a tree structure of depth $\lceil\log_2(M)\rceil$ with $M-1$ spike merger elements, where $M$ is the number of separate spike trains entering the network.
\\
While this merger element allows finding the spike that occurs next, it obfuscates the information of which spike train produced that spike. This information is necessary for selecting the weights that will later be added to the respective neuron potentials, since each spike train is assigned a separate set of weights during the training stage. Fortunately, finding the producing spike train can be solved relatively easily by introducing a separate vector which is extended by one bit during each stage of the $\log_2(M)$ tree. Depending if the upper or lower spike train produced the minimum for each spike merger element, either a $1$ or a $0$ is appended to the vector, thus producing an index identifying the originating spike train.
\subsection{Processing the neuron layer}
The spike merger tree is only needed once for the entire network, and receives all of the input spike trains. The produced differential time to the next spike, as well as the index of the producing spike train, i.e. the synapse index, are then sent into the first layer of the network. The synapse index can be used to directly address the weights in memory, providing the weights of all the synapses in the next layer which are connected to the synapse that produced the input spike. All retrieved weights are then handed in parallel to the respective neuron cores, which hold the neuron potential for every neuron in the layer. These weights are added directly to the neuron potential registers after which the thresholding operation and spike generation is handled by the neuron cores.
\\
The layer controller in turn is responsible for processing the differential time produced by either the previous layer or the spike merger tree, in order to generate the control signals necessary for decaying the neuron potential in each core. Furthermore, it is also responsible for ensuring that the neuron cores generate spikes at the correct clock cycle, guaranteeing proper timing between layers, which also includes delaying the forwarding of the differential time until all spikes of the current layer have been processed (for \ac{lif} neurons, output spikes can only be triggered at times of input spikes). As each of the neuron cores produces only a single bit of information ($0\rightarrow$ no spike, $1 \rightarrow$ spike), these bits are then fed into a \ac{lopd}~\cite{fast_lp_lod,vlsi_lod,vlsi_lp_lod,approx_lod}, to once again arrive back at the index of a synapse that triggered a spike, which can be fed into the next layer. Therefore, each of the layers present in the network is directly instantiated in the hardware and can work in parallel to every other layer, as shown in \figref{fig:hw_arch}.
\begin{figure*}
    \centering
    \includegraphics[width=0.8\linewidth]{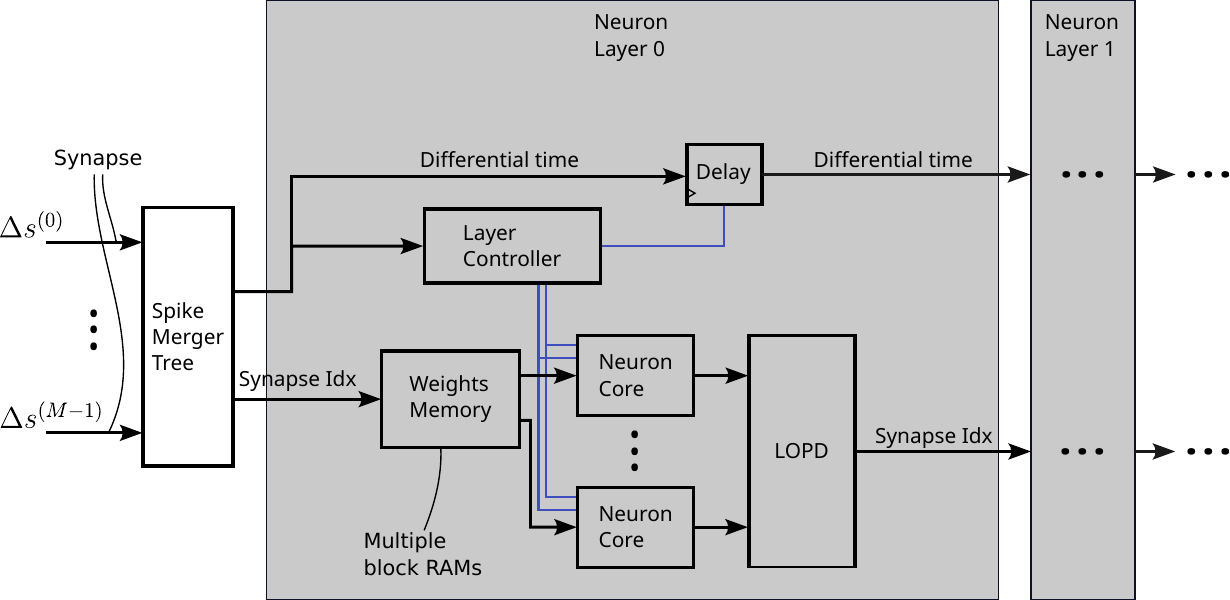}
    \caption{Hardware architecture capable of accelerating arbitrarily deep and wide feedforward networks of \ac{lif} neurons~\cite{iscas2024}.}
    \label{fig:hw_arch}
\end{figure*}

\section{Synthesis Results}
Complementary to the \ac{fpga} synthesis results that will be shown below, an ASIC synthesis run was also completed using the open-source OpenROAD software~\cite{openroad} and the accompanying \acp{orf}. While there are quite a few open-source \acp{pdk} available, such as the SKY130 and IHP 130nm \ac{pdk}~\cite{ihp_pdk}, there are now also predictive \acp{pdk} available for smaller technology nodes. One is the openly available ASAP7 \ac{pdk}~\cite{ASAP7}, which was used to run the synthesis for the design shown in \figref{fig:hw_arch}. While the \acp{orf} already provide built-in support for the ASAP7 \ac{pdk}, a lot of other tools are necessary to complete the flow, e.g. KLayout and Yosys among others, all of which are thankfully provided in many pre-built Docker containers. The one chosen here is the container hosted by the Institute for Integrated Circuits and Quantum Computing at the Johannes Kepler University, which can be found in~\cite{iic_osic_tools}. 
\\\\
Apart from initially setting up the flow, another big issue compared to the \ac{fpga} implementation is the use of blockrams for storing the weights of the network. While readily available in modern \acp{fpga}, synthesizing them using regular registers would not only be wasteful in terms of area and power, but has also resulted in non-terminating runtimes of the OpenROAD software, potentially due to diverging optimization runs during floorplanning and place \& route optimization steps. Therefore, it is necessary to replace the blockrams with dedicated SRAM blocks, which are publicly available for the ASAP7 technology node~\cite{asap7_sram}, although they have not been in use widely, according to discussions with the maintainers. While the \ac{fpga} implementation can easily access multiple blockrams in parallel to achieve very wide, i.e. high-bit-width, access patterns, for the ASIC this means instantiating multiple SRAMs in order to achieve the same bit width. Theoretically it is possible to build SRAM blocks with very wide access patterns, however none of the open-source software tools available at the time of writing were capable of generating such SRAM blocks for the ASAP7 \ac{pdk}. While it is possible to add logic to simply collect the necessary data over several clock cycles from a smaller number of SRAM blocks, this would have lessened the comparability of the ASIC to the FPGA implementation since quite a few modifications to the RTL would have been required. Therefore, it was chosen to simply instantiate multiple SRAM blocks in parallel, in order to achieve the necessary bit widths, which obviously leads to less than ideal floorplanning and routing, thus negatively influencing the performance and power requirements of the chip. The final layout of the entire chip can be seen in \figref{fig:asic}, with the respective areas and color coding shown in \tabref{tab:asic_area}. Note that we consider the results for the ASIC synthesis in terms of clock speed and area, as informative bounds. With optimized SRAM blocks we expect that the clock speed can be significantly increased and the area further reduced. This also becomes evident when looking at the power consumption for the ASIC shown in \tabref{tab:asic_energy}, which highlights the fact that the macros, i.e. the SRAM blocks, are responsible for more than $67\%$ of the total power consumption of the chip, as well as $\approx30\%$ of the total area. In contrast, the entirety of the combinational and sequential design uses up a mere $\approx3.7\%$ of the entire area and $\approx{}18.7\%$ of the power (disregarding the power used for distributing the clock). Note that the ASIC flow was only completed for the accelerator structure optimized for the MNIST dataset, since the Cifar-10 version would have required even more SRAM blocks, which the current verson of the flow scripts did not allow to be placed.

\begin{figure}
    \includegraphics[width=\linewidth]{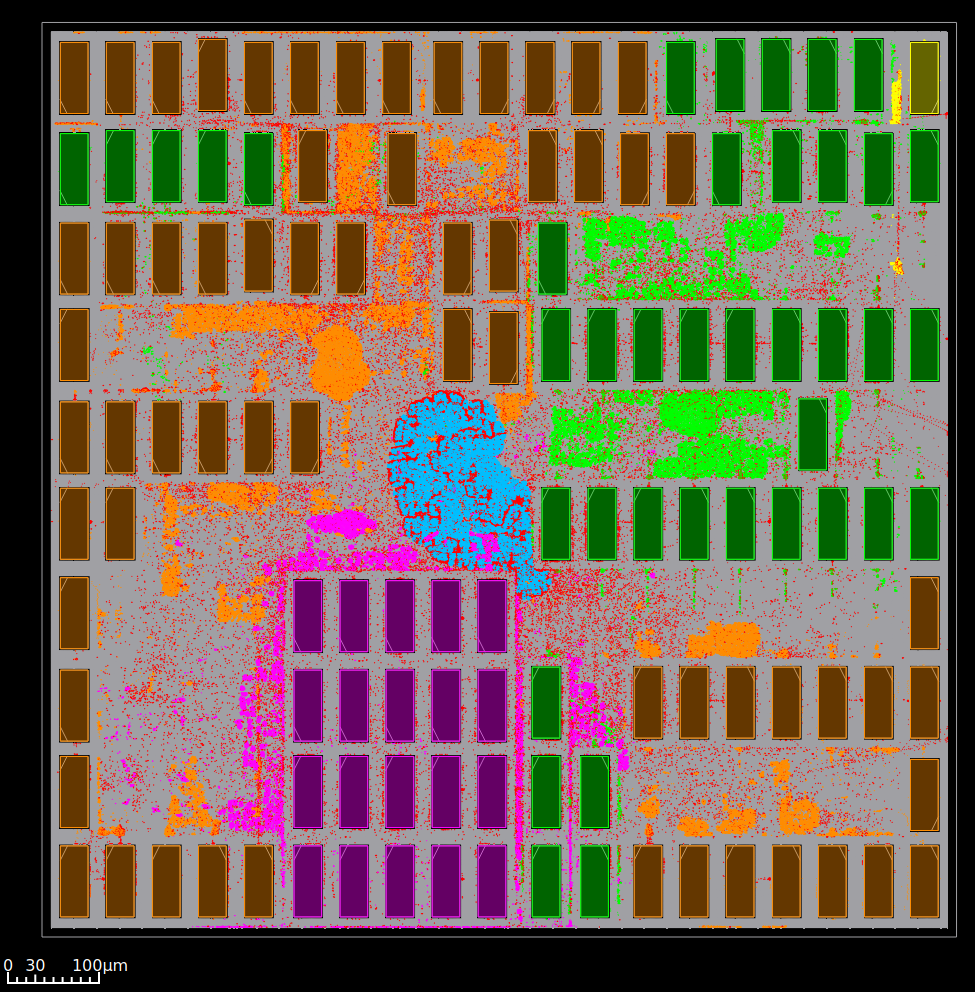}
    \caption{ASIC layout using open-source \acp{orf} and ASAP7 \ac{pdk} with 124~SRAM blocks of size 64x1024 bits for a fully connected feedforward network of size 400-800-512-256-10. The color coding for the different blocks can be seen in \tabref{tab:asic_area}.}
    \label{fig:asic}
\end{figure}
\bgroup
\def\arraystretch{1.2}
\begin{table}
    \centering
    \caption{Color coding and respective area utilization for each component of \figref{fig:asic}.}
    \begin{tabular}{llrr}
        \textbf{Color}            & \textbf{Component}& \multicolumn{2}{c}{\textbf{Area}} \\\hline
        \cellcolor{red}~          & Top level         &   1.601e-03 $mm^2$ & 0.160\%    \\
        \cellcolor{min_network}~  & Spike Merger Tree &   7.866e-03 $mm^2$ & 0.787\%    \\
        \cellcolor{layer400_800}~ & Layer 400-800     &  13.764e-03 $mm^2$ & 1.376\%    \\
        \cellcolor{layer800_512}~ & Layer 800-512     &   8.872e-03 $mm^2$ & 0.887\%    \\
        \cellcolor{layer512_256}~ & Layer 512-256     &   4.685e-03 $mm^2$ & 0.469\%    \\
        \cellcolor{layer256_10}~  & Layer 256-10      &   0.259e-03 $mm^2$ & 0.026\%    \\
                                  & SRAM blocks       & 292.623e-03 $mm^2$ & 29.260\%   \\\hline
                                  & Total             & 328.069e-03 $mm^2$ & 32.807\%
    \end{tabular}
    \label{tab:asic_area}
\end{table}
\egroup

\bgroup
\def\arraystretch{1.2}
\begin{table}
    \centering
    \caption{Power requirements for the ASAP7 chip running at 538MHz.}
    \begin{tabular}{p{1.5cm}|p{1cm}p{1cm}p{1cm}p{1cm}|p{0.7cm}}
                      & Internal [W] & Switching [W] & Leakage [W] & Total [W]  & \\\hline
        Sequential    &        2.53e-02    & 4.74e-04            & 6.97e-06    &  2.58e-02  &   9.8\% \\
        Combinational &        8.93e-03    & 1.46e-02            & 5.36e-05    &  2.36e-02  &   8.9\% \\
        Clock         &        2.20e-02    & 1.39e-02            & 1.53e-06    &  3.59e-02  &  13.6\% \\
        Macro         &        1.79e-01    & 0.00e-00            & 0.00e-00    &  1.79e-01  &  67.7\% \\\hline
        Total         &        2.35e-01    & 2.89e-02            & 6.21e-05    &  2.64e-01  & 100.0\% \\
    \end{tabular}
    \label{tab:asic_energy}
\end{table}
\egroup

\bgroup
\def\arraystretch{1.3}

\begin{table*}
	\centering
    \begin{threeparttable}
	\caption{Comparison to other hardware accelerator architectures.}
	\begin{tabular}{lrrrrrrrr}
        Work                   & Process          & Dataset  & Accuracy [\%]  & Inferences/sec & Clock freq. [MHz] & Area [mm²]    & \#Multipliers & Power [W] \\\hline
        This work              & ASAP7            & MNIST    & \textbf{99.03} & $\approx6000$  & \textbf{538}      & \textbf{1}    & \textbf{0}    & 0.26 \\
        \cite{sparse_snn_vlsi} & 10nm FinFET CMOS & MNIST    & 97.9           & \textbf{6250}  & 509               & 1.99          & N/R$^\dagger$ & 0.006$^\ast$ \\
        \cite{spinnaker_mnist} & 130nm L130E      & MNIST    & 95.01          & $\approx 50$   & 150               & 102           & N/R$^\dagger$ & $\approx0.5$\\
        \cite{sub_uw_snn}      & 65nm LP CMOS     & MNIST    & 97.6           & 2              & 0.07              & 1.72          & N/R$^\dagger$ & \textbf{0.3e-6} \\\\

        Work                       & FPGA    & Dataset  & Accuracy [\%]  & Inferences/sec         & Clock freq. [MHz] & \#LUTs         & \#DSP blocks  & Power [W] \\\hline
        This work~\cite{iscas2024} & ZCU102  & MNIST    & 99.03          & $\approx$ \textbf{3400} & \textbf{300}      & \textbf{56520} & \textbf{0}    & \textbf{2.44} \\
        \cite{syncnn}              & ZCU102  & MNIST    & \textbf{99.60} & 1631                   & 200               & 224690         & 555           & N/R$^\dagger$ \\
        \cite{fang}                & XCZU9EG & MNIST    & 99.20          & 2124                   & 125               & 155952         & 1795          & 4.5 \\
        This work                  & ZCU102  & Cifar-10 & 55.15          & $\approx2500$          & 252               & 149445         & \textbf{0}    & 6.63 \\
        \cite{syncnn}              & ZCU102  & Cifar-10 & 90.8           & 62                     & 200               & N/R$^\dagger$  & N/R$^\dagger$ & N/R$^\dagger$ \\
    \end{tabular}
	\label{tab:perf_comparison}
    \begin{tablenotes}
        \small
        \item[$\dagger$] Not reported 
        \item[$\ast$] Uses separate power domain with different voltage and clock gating of unused SRAM cells to reduce power 
        \end{tablenotes}
    \end{threeparttable}
\end{table*}
\egroup

\section{Learning the Encoding}
\label{sec:learned_encoding}
The hardware described in \secref{sec:hw_arch} processes spikes given in differential time, however, this representation does not match the data that is often given as equidistantly spaced vector data, such as pixel maps obtained from cameras. Although there are approaches for sensors that directly produce spikes, such as event-based cameras or level-crossing ADCs~\cite{samplingproto,nw_lc_adc,window_lc_adc}, usually the equidistantly sampled data needs to be converted into spike trains, a process that is referred to as encoding the data. While many different encoding methods have been proposed, with some of them having even been integrated into learning frameworks like SNNTorch~\cite{snntorch}, it is the responsibility of the user to choose the best encoding method for the given dataset and problem. The authors of this work propose to learn this encoding in conjunction with the network weights during the training process of the network. For image based datasets such as MNIST~\cite{mnist} or Cifar-10~\cite{cifar10} this is accomplished by selecting square patches from the pixels, which are then serialized. The serialized pixel values are then immediately fed into a \ac{lif} neuron that is responsible for the encoding, as can be seen in \figref{fig:patch}. Inside the network, typically the inputs to the \ac{lif} neuron are in the set $\{0,1\}$, thus making multiplication with the weights unnecessary. However, when the pixel values are directly used as inputs, learning the weights would require multiplications, which is undesirable from an area and power perspective if the encoding is to be included in the accelerator. Therefore, the weights to be learned by the network have to be restricted to the set $\{-1, 0, 1\}$, thus only requiring addition and subtraction. While this restriction does negatively influence the accuracy achieved with the encoding, the difference is almost negligible according to our results (see \cite{iscas2024}). The results when using the learned encoding on the MNIST and Cifar-10 test datasets respectively, can be seen in \tabref{tab:perf_comparison}. 
\\
As one can see from the ASIC synthesis results, the proposed design is superior in terms of accuracy for the MNIST data set, while maintaining low area / resource requirements. Although the energy requirements are higher than the comparison works, we also incorporate the power of the SRAM blocks into our calculations (one could further use methods from~\cite{sparse_snn_vlsi} to reduce the power consumption of the memory blocks as well). For the compared FPGA implementations, the proposed work is the fastest, both achieving the highest clock frequency and also the largest number of inferences/second while requiring no DSP blocks and using the least power. Meanwhile, the accuracy results for the proposed implementation is close to the top results of the compared works. Although the results for the Cifar-10 dataset are well below the maximum reported accuracies, they coincide with typical accuracy values for feedforward \acp{ann}~\cite{cifar10_fc,cifar10_fc_github}. As a future task, it remains to be investigated how to incorporate convolutional-like layers in \acp{snn} with low complexity while maintaining high accuracy.

\begin{figure}
    \includegraphics[width=\linewidth]{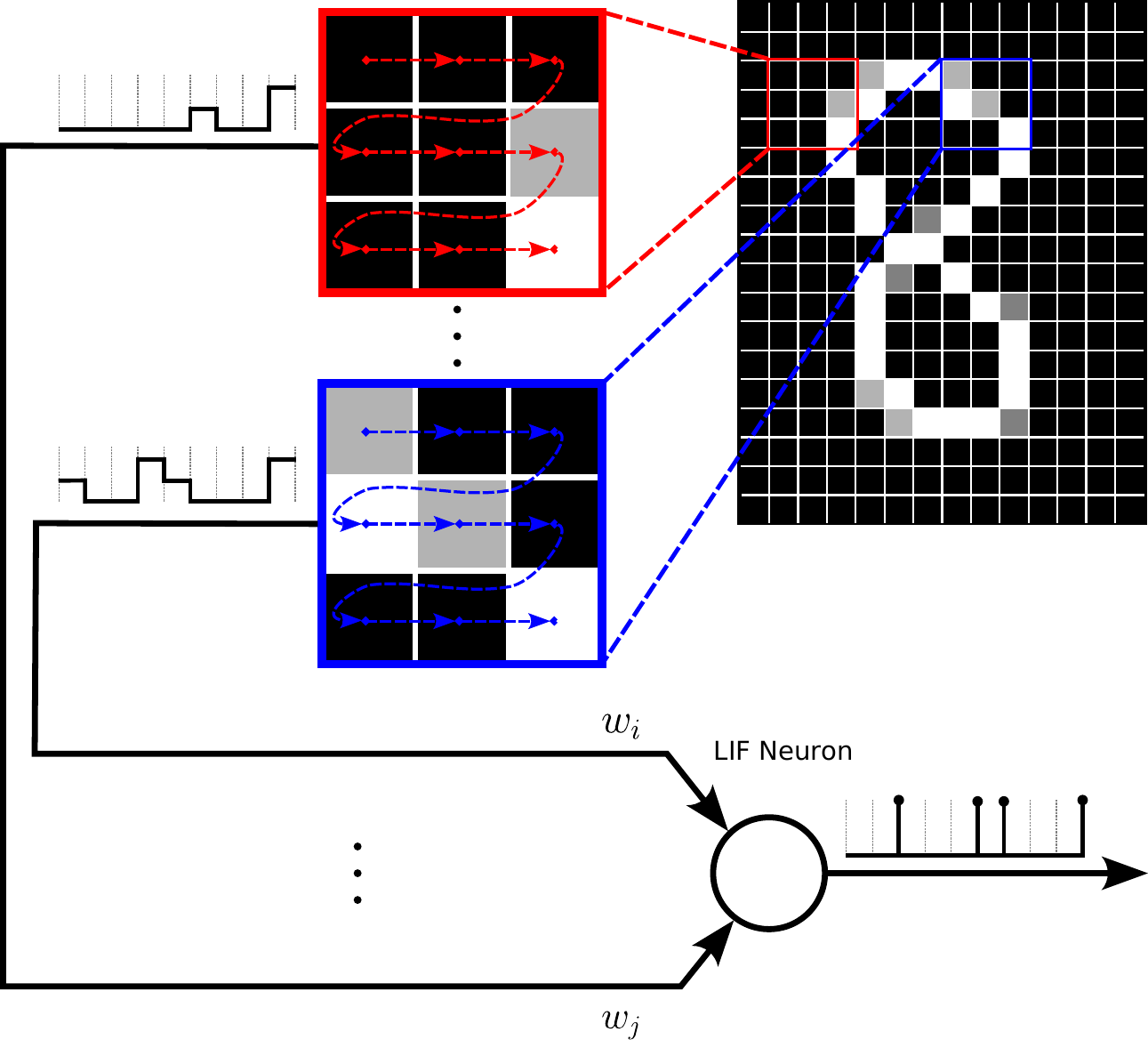}
    \caption{3x3 learned encoding by using patch-wise selection of the pixels with subsequent neuron activation~\cite{iscas2024}.}
    \label{fig:patch}
\end{figure}

\section{Conclusion}
This work discusses different ways for representing spike trains and their respective advantages and disadvantages. It establishes that the previously presented differential time encoding used to encode spike trains has many advantages compared to the absolute time and sampled representation. The spike merger is introduced as a way to join disparate spike trains, with an accompanying efficient hardware design, which serves as the entrypoint to the presented \ac{snn} accelerator architecture. Synthesis results for a Xilinx Ultrascale+ FPGA are shown, as well as ASIC design results using the open-source 7nm ASAP7 PDK alongside a fully open-source software stack and flow. These results clearly show that the design achieves superior accuracy on the MNIST dataset compared to other works, while maintaining low resource requirements.

\section*{Acknowledgments}
This research received funding from the Austrian Research Promotion Agency (FFG) and the Austrian ministry for Climate Action, Environment, Energy, Mobility, Innovation and Technology (BMK) under project no. FO999899263 and funding from the Chips Joint Undertaking under project no. 101112274 (ISOLDE) and its members Austria, Czechia, France, Germany, Italy, Romania, Spain, Sweden, Switzerland. 
\\\\
This work is furthermore supported by: the COMET-K2 ``Center for Symbiotic Mechatronics'' of the Linz Center of Mechatronics (LCM), funded by the Austrian federal government and the federal state of Upper Austria. 

\clearpage
\bibliographystyle{IEEEtran}
\balance
\bibliography{bibliography.bib}

\vfill

\end{document}